\documentclass[a4paper,UKenglish,cleveref, autoref, thm-restate]{lipics-v2021}

\usepackage{arydshln}
\usepackage{tikz}

\title{Absolute Expressiveness of Subgraph-based Centrality Measures}

\author{Andreas Pieris}{University of Edinburgh, UK \and University of Cyprus, Cyprus}{apieris@inf.ed.ac.uk}{https://orcid.org/0000-0003-4779-3469}{}

\author{Jorge Salas}{Pontificia Universidad Cat\'olica de Chile, Chile \and University of Edinburgh, UK}{jusalas@uc.cl}{https://orcid.org/0000-0001-5535-3055}{}

\authorrunning{A. Pieris and J. Salas}

\Copyright{Andreas Pieris and Jorge Salas}



\relatedversion{} %optional, e.g. full version hosted on arXiv, HAL, or other respository/website
\acknowledgements{Pieris was supported by the EPSRC grant EP/S003800/1 ``EQUID''. 
Salas was supported by ANID -- Millennium Science Initiative Program -- Code ICN17\_002.}

\nolinenumbers %uncomment to disable line numbering

\EventEditors{Floris Geerts and Brecht Vandevoort}
\EventNoEds{2}
\EventLongTitle{26th International Conference on Database Theory (ICDT 2023)}
\EventShortTitle{ICDT 2023}
\EventAcronym{ICDT}
\EventYear{2023}
\EventDate{March 28--31, 2023}
\EventLocation{Ioannina, Greece}
\EventLogo{}
\SeriesVolume{255}
\ArticleNo{8}
\begin{document}

\maketitle

\renewcommand{\paragraph}[1]{\textbf{#1}}

\newcommand{\mi}[1]{\mathit{#1}}
\newcommand{\ins}[1]{\mathbf{#1}}
\newcommand{\sub}{\text{\rm Sub}}
\newcommand{\dom}{\text{\rm Dom}}
\newcommand{\gr}{\mathbf{G}}
\newcommand{\vg}{\mathbf{VG}}
\newcommand{\vcg}{\mathbf{VCG}}
\newcommand{\congr}{\mathbf{CG}}
\newcommand{\comp}{\text{\rm Comp}}
\newcommand{\Ce}{\mathsf{C}}
\newcommand{\conC}{\mathsf{ConC}}
\newcommand{\con}{\mathsf{Con}}
\newcommand{\F}{\mathsf{F}}
\newcommand{\A}{\mathsf{A}}
\newcommand{\ps}{\mathcal{P}}
\newcommand{\sd}{\textsc{SD}}
\newcommand{\subc}[1]{\mathsf{C}\langle #1 \rangle}
\newcommand{\rank}{\text{\rm Rank}}
\newcommand{\val}[2]{\text{\rm Val}_{#1}^{#2}}
\newcommand{\bval}[2]{\text{\rm BVal}_{#1}^{#2}}
\newcommand{\pc}{\mathit{pc}}

\newcommand{\stress}{\mathsf{Stress}}
\newcommand{\allsub}{\mathsf{All\text{\rm -}Subgraphs}}
\newcommand{\degree}{\mathsf{Degree}}
\newcommand{\clique}{\mathsf{Cross\text{\rm -}Clique}}
\newcommand{\closeness}{\mathsf{Closeness}}  
\newcommand{\decay}{\mathsf{Decay}}
\newcommand{\harmonic}{\mathsf{Harmonic}}
\newcommand{\pagerank}{\mathsf{PageRank}}
\newcommand{\eigenvector}{\mathsf{Eigenvector}}
\newcommand{\betweeness}{\mathsf{Betweenness}}

%%%%%%%%%%%%%%%%%%%%%%%%% Environment delimiters

\def\qed{\hfill{\qedboxempty}      % qed with empty box
  \ifdim\lastskip<\medskipamount \removelastskip\penalty55\medskip\fi}

\def\qedboxempty{\vbox{\hrule\hbox{\vrule\kern3pt
                 \vbox{\kern3pt\kern3pt}\kern3pt\vrule}\hrule}}

\def\qedfull{\hfill{\qedboxfull}   % qed with full box
  \ifdim\lastskip<\medskipamount \removelastskip\penalty55\medskip\fi}

\def\qedboxfull{\vrule height 4pt width 4pt depth 0pt}

\newcommand{\markfull}{\qedboxfull}
\newcommand{\markempty}{\qed} 

%% COMMENTS

\newcommand{\jorge}[1]{\todo[inline, color=blue!15]{{\bf Jorge:} #1}}
\newcommand{\andreas}[1]{\todo[inline, color=red!30]{{\bf Andreas:} #1}}

%% UNCOMMENT THIS TO REMOVE COMMENTS FROM THE PAPER
%\renewcommand{\jorge}[1]{}
%\renewcommand{\cristian}[1]{}
%\renewcommand{\oskar}[1]{}

\begin{abstract}
	In graph-based applications, a common task is to pinpoint the most important or ``central'' vertex in a (directed or undirected) graph, or rank the vertices of a graph according to their importance. To this end, a plethora of so-called centrality measures have been proposed in the literature. Such measures assess which vertices in a graph are the most important ones by analyzing the structure of the underlying graph. A family of centrality measures that are suited for graph databases has been recently proposed by relying on the following simple principle: the importance of a vertex in a graph is relative to the number of ``relevant'' connected subgraphs surrounding it; we refer to the members of this family as subgraph-based centrality measures. Although it has been shown that such measures enjoy several favourable properties, their absolute expressiveness remains largely unexplored. The goal of this work is to precisely characterize the absolute expressiveness of the family of subgraph-based centrality measures by considering both directed and undirected graphs. To this end, we characterize when an arbitrary centrality measure is a subgraph-based one, or a subgraph-based measure relative to the induced ranking. These characterizations provide us with technical tools that allow us to determine whether well-established centrality measures are subgraph-based. Such a classification, apart from being interesting in its own right, gives useful insights on the structural similarities and differences among existing centrality measures.
\end{abstract}

\section{Introduction}\label{sec:introduction}
%

%Graph theory is a very old subject (its roots can be traced back to the 18th century) with many modern applications. 
%The simplicity and adaptability of graphs make them a prominent modelling tool in Computer Science. Importantly, 
Graphs are well-suited for representing complex networks such as biological networks, cognitive and semantic networks, computer networks, and social networks, to name a few.
% considering distinct objects represented by vertices (also called nodes), and the connections between those objects as edges. 
%
In many applications that involve (directed or undirected) graphs, a crucial task is to pinpoint the most important or ``central'' vertex in a graph, or rank the vertices of a graph according to their importance. Indeed, these graph-theoretic tasks naturally appear in many different contexts, for example, finding people who are more likely to spread a disease in the event of an epidemic~\cite{DeBa02}, highlighting cancer genes in proteomic data~\cite{GaVi10}, assessing the importance of websites by search engines~\cite{page1999pagerank}, identifying influencers in social networks~\cite{HCWR20}, and many more.
%; this list is by no means exhaustive.
%
To this end, a plethora of centrality measures have been proposed that assess the importance of a vertex in a graph~\cite{BoEv06,newman2018networks}. Centrality measures have been also studied in a principled way with the aim of providing axiomatic characterizations via structural properties over certain classes of graphs; see, e.g.,~\cite{Kitti:2016,Brink:Gilles:2000,was2018axiomatization}.

%However, our work differ from these in how we answer axiomatization. Our characterization does not depend on structural properties over particular graphs, but on how the centrality measure assign values related to subgraphs amount surrounding a node. This is the key reason why we can extend our results beyond undirected graphs as long as a meaningful definition of motifs exists.

%by analyzing the structure of the underlying graph.
%
%As a simple example, we recall a well-known centrality measure over undirected graphs, called stress centrality, introduced in the 1950s in the context of communication networks~\cite{Shim53}: it measures the centrality of a vertex by counting the number of shortest paths that go via that vertex; the formal definition is given in Section~\ref{sec:preliminaries}.

%\medskip
%\noindent \paragraph{Centrality Measures and Graph-structured Data.}

It is not surprising that centrality measures have been also considered in the context of graph-structured data.
%As said above, measuring the importance of a vertex in a graph is an important task in many graph-based applications, and graph-structured data is no exception. 
%
Major graph database management systems such as
Neo4j\footnote{https://neo4j.com/docs/graph-data-science/current/algorithms/centrality/} and TigerGraph\footnote{https://docs.tigergraph.com/graphml/current/centrality-algorithms/}, have already adopted and implemented several centrality measures and algorithms in their Graph Data Science library such as Eigenvector~\cite{bonacich1987power}, PageRank~\cite{page1999pagerank}, Closeness~\cite{Sabi66}, and many others.
Moreover, applications of centrality measures have recently emerged in the context of knowledge graphs for entity linking~\cite{MHL20}, and Semantic Web search engines where ranking results is a central task~\cite{HHUKPD11}.

Several existing centrality measures rely on the following intuitive principle: the importance of a vertex in a graph is relative to the number of {\em connected subgraphs} (e.g., triangles, paths, or cliques) surrounding it. We refer to such measures as {\em subgraph-based}.
%
%By subgraph motif we mean any recurrent and statisitically significant connected subgraph (e.g., triangles, paths, or cliques), which models some specific interactions inside graphs~\cite{HoLe74}. Hence, it is indeed natural to expect that the importance of a vertex in a graph is relative to the number of subgraph motifs surrounding it; see, e.g.,~\cite{everett1998analyzing,freeman1977set,KoSS07,Shim53}.
%
Interestingly, subgraph-based centrality measures are of particular interest for graph-structured data since a connected subgraph can be understood as the potential graph patterns occurring in a graph database.  
Consider, for example, a property graph $G$, which is essentially a finite directed graph, and a language $L$ of basic graph patterns~\cite{AABHRV17}. The evaluation of a query $Q$ from $L$ over $G$, denoted $Q(G)$, is the set of vertices of $G$ that comply with the graph pattern expressed by $Q$.
It is reasonable to assume that the more queries $Q$'s from $L$ such that $v \in Q(G)$ exist, the more important $v$ is in $G$ (relative to $L$).
This way of defining the importance of a vertex follows the general principle discussed above, where the relevant connected subgraphs are the basic graph patterns from the language $L$.
%
%In other words, a monotonic subgraph motif measure, where the subgraph motifs are chosen to be the graph patterns of $L$, is a measure that is naturally induced by $L$. 
%
%Analogously, in a different scenario where the less queries $Q$ from $L$ such that $v \in Q(G)$, the more important $v$ is in $G$ (relative to $L$), a subgraph motif measure where the underlying subgraph motifs are the graph patterns of $L$, and the filtering function is a decreasing one, could  be adopted.
%
%In view of this, it is indeed fair to claim that subgraph motif measures are of particular interest for graph-structured data.

A framework for defining and studying subgraph-based centrality measures has been recently introduced by Riveros and Salas~\cite{RiSa20}, where the importance of a vertex is defined as the logarithm of the number of connected subgraphs surrounding it.
As explicitly discussed in~\cite{RiSa20}, the choice of applying the logarithmic function is purely for technical simplicity, and one could adopt any function, which we call filtering function, that leads to a richer family of subgraph-based centrality measures.
Note that~\cite{RiSa20} considered only undirected graphs, but we can naturally define subgraph-based centrality measures over directed graphs.
The main outcome of the analysis performed in~\cite{RiSa20} is that subgraph-based centrality measures satisfy desirable theoretical properties, typically called axioms, provided that the underlying family of connected subgraphs enjoys certain properties.

Despite the thorough analysis performed in~\cite{RiSa20}, the absolute expressiveness of the family of subgraph-based centrality measures remains largely unexplored.
Our main objective is to delineate the limits of the family of subgraph-based measures for both directed and undirected graphs. More precisely, we would like to understand when an arbitrary centrality measure is a subgraph-based one, or when it induces the same ranking as a subgraph-based one.

\medskip

\noindent \paragraph{Our Contributions.} Our contributions can be summarized as follows:

\begin{itemize}
\item In Section~\ref{sec:characterization}, we provide a precise characterization of when an arbitrary centrality measure is subgraph-based. More precisely, we isolate a ``bounded value'' property $P$ over centrality measures, which essentially states that the total number of distinct values that can be assigned to vertices surrounded by a certain number of connected subgraphs is bounded, and then show that a measure can be expressed as a subgraph-based one iff it enjoys $P$.

\item We then proceed in Section~\ref{sec:characterization-ranking} to characterize when an arbitrary centrality measure induces the same ranking as a subgraph-based measure. In this case, we isolate a ``graph coloring'' property $P$ over centrality measures, and then show that a centrality measure can be expressed as a subgraph-based one relative to the induced ranking iff it enjoys $P$.

\item In Section~\ref{sec:monotonic-functions}, we focus on the family of monotonic subgraph-based measures, i.e., subgraph-based measures with a monotonic filtering function, and provide analogous characterizations via refined properties in the spirit of the ``bounded value'' property discussed above.
An interesting finding is that in the case of connected graphs, {\em every} centrality measure can be expressed as a monotonic subgraph-based measure  relative to the induced ranking.

\item We finally proceed in Section~\ref{sec:classification} to determine if established measures (such as PageRank, Eigenvector, and many others) are (monotonic) subgraph-based (relative to the induced ranking). Such a classification, apart from being interesting in its own right, provides insights on the structural similarities and differences among the considered measures.
\end{itemize}

\smallskip
\noindent\paragraph{Clarification Remark.} In the rest of the paper, due to space constraints and for the sake of clarity, we focus on undirected graphs, but all the notions and results can be transferred to the case of directed graphs under the standard notion of weak connectedness. 
%Recall that a directed graph is weakly connected if the underlying undirected graph, obtained by simply ignoring the directions, is connected. 
%The case of directed graphs is deferred to the appendix.
\section{Preliminaries}\label{sec:preliminaries}

We recall the basics on undirected graphs and graph centrality measures. In the rest of the paper, we assume the countable infinite set $\ins{V}$ of {\em vertices}. For $n>0$, let $[n] = \{1,\ldots,n\}$.

\medskip

\noindent
\paragraph{Undirected Graphs.} An {\em undirected graph} (or simply {\em graph}) $G$ is a pair $(V,E)$, where $V$ is a finite non-empty subset of $\ins{V}$ (the set of {\em vertices of $G$}), and $E \subseteq \{\{u,v\} \mid u,v \in V\}$ (the set of {\em edges of $G$}). For notational convenience, given a graph $G$, we write $V(G)$ and $E(G)$ for the set of its vertices and edges, respectively.
We denote by $\gr$ the set of all graphs, and by $\vg$ the set of vertex-graph pairs $\{(v,G) \in \ins{V} \times \gr \mid v \in V(G)\}$.
The {\em neighbourhood} of a vertex $v \in V(G)$ in $G$, denoted $N_G(v)$, is the set $\{u \in V(G) \mid \{u,v\} \in E(G)\}$. For $u \in N_G(v)$, we say that $v$ and $u$ are {\em adjacent} in $G$.
For a vertex $v \in \ins{V}$, we write $G_v$ for the graph $(\{v\},\emptyset)$.

A {\em subgraph} of a graph $G$ is a graph $G'$ such that $V(G') \subseteq V(G)$ and $E(G') \subseteq E(G)$; we write $G' \subseteq G$ to indicate that $G'$ is a subgraph of $G$. Note that the binary relation $\subseteq$ over graphs forms a partial order. We denote by $\sub(G)$ all the subgraphs of $G$, that is, the set of graphs $\{G' \mid G' \subseteq G\}$.
Given a set of vertices $S \subseteq V(G)$, the subgraph of $G$ {\em induced} by $S$, denoted $G[S]$, is the subgraph $G'$ of $G$ such that $V(G') = S$ and $E(G') = \{\{u,v\} \in E(G) \mid u,v \in S\}$.

A {\em path} in $G$ is a sequence of vertices $\pi = v_0,v_1,\ldots,v_n$, for $n \geq 0$, such that $\{v_i,v_{i+1}\} \in E(G)$ for every $ 0 \leq i < n$. We further say that $\pi$ is a path from $v_0$ to $v_n$. The length of $\pi$, denoted $|\pi|$, is the number of edges in $\pi$, i.e., $n$. By convention, there exists a path of length $0$ from a vertex to itself.
The {\em distance} between two vertices $u,v \in V(G)$ in $G$, denoted $d_G(u,v)$, is defined as the length of a shortest path from $u$ to $v$ in $G$; if there is no path, then $d_G(u,v) = \infty$.
We denote by $S_G(u,v)$ the set of all the shortest paths from $u$ to $v$ in $G$, that is, the set $\{\pi \mid \pi \text{ is a path from } u \text{ to } v \text{ in } G \text { with } |\pi| = d_G(u,v) \}$.

A graph $G$ is {\em connected} if, for every two distinct vertices $u,v \in V(G)$, there exists a path from $u$ to $v$.
We denote by $\A(v,G)$ the set of all connected subgraphs of $G$ that contain $v$, that is, the set $\{G' \subseteq G \mid v \in V(G') \text{ and } G' \text{ is connected}\}$.
By abuse of notation, we may treat $\A(\cdot,\cdot)$ as a function of the form $\vg \rightarrow \mathcal{P}(\gr)$; as usual, $\ps(S)$ denotes the powerset of a set $S$.
A {\em connected component} (or simply {\em component}) of $G$ is an induced subgraph $G[S]$ of $G$, where $S \subseteq V(G)$, such that $G[S]$ is connected, and, for every $v \in V(G) \setminus S$, there is no path in $G$ from $v$ to a vertex of $S$.
It is clear that whenever $G$ is connected, the only component of $G$ is $G$ itself. We denote by $\comp(G)$ all the components of $G$, that is, the set of graphs $\{G' \mid G' \text{ is a component of } G\}$.
Let $K_v(G)$ be the set of vertices of the component of $G$ containing the vertex $v$.

Two graphs $G_1$ and $G_2$ are {\em isomorphic}, denoted $G_1\simeq G_2$, if there exists a bijective function $h: V(G_1) \to V(G_2)$ such that $\{v,u\} \in E(G_1)$ iff $\{h(v),h(u)\}\in E(G_2)$.
Furthermore, given the vertices $v_1 \in V(G_1)$ and $v_2 \in V(G_2)$, we say that the pairs $(v_1, G_1)$ and $(v_2,G_2)$ are isomorphic, denoted $(v_1, G_1) \simeq (v_2,G_2)$, if $G_1\simeq G_2$ witnessed by $h$ and $h(v_1)= v_2$.

\medskip

\noindent 
\paragraph{Centrality Measures.} A centrality measure assigns a score to a vertex $v$ in a graph $G$, which reflects the importance of $v$ in $G$. In other words, we adopt the standard assumption that the higher the score of a vertex $v$ in $G$, the more important or ``central'' $v$ is in $G$.
Furthermore, it is typically assumed that the values assigned by a measure to the vertices of a graph do not depend on the names of the vertices, but only on the structure of the graph. In other words, two isomorphic vertices occurring in isomorphic graphs should be assigned the same importance; the latter property is usually called {\em closure under isomorphism} or {\em anonymity}. The formal definition of the notion of centrality measure follows:

\begin{definition}[\textbf{Centrality Measure}]\label{def:centrality-measure}
	A {\em centrality measure} (or simply {\em measure}) is a function $\Ce : \vg\ \rightarrow\ \mathbb{R}$ such that, for every two pairs $(v_1,G_1) \in \vg$ and $(v_2,G_2) \in \vg$, $(v_1, G_1) \simeq (v_2,G_2)$ implies $\Ce(v_1,G_1) = \Ce(v_2,G_2)$. \hfill\markfull
\end{definition}

We proceed to recall three known centrality measures that will be used throughout the paper; more centrality measures are discussed in Section~\ref{sec:classification}.

\begin{description}
	\item[\textit{Stress.}] This is a well-known centrality measure introduced in the 1950s~\cite{Shim53}. It measures the centrality of a vertex by counting the number of shortest paths that go via that vertex. For a graph $G$ and a vertex $v \in V(G)$, let $S_{G}^{v}(u,w)$ be the set of paths $\{\pi \in S_{G}(u,w) \mid \pi \text{ contains } v\}$. The {\em stress centrality} of $v$ in $G$ is defined as follows:
	\[
	\stress(v,G)\ =\ \sum_{u,w \in V(G) \setminus \{v\}} \left|S_{G}^{v}(u,w)\right|.
	\]

	\item[\textit{All-Subgraphs.}] This measure was recently introduced in the context of graph databases~\cite{RiSa20}. It states that a vertex is more central if it participates in more connected subgraphs. Formally, given a graph $G$ and a vertex $v \in V(G)$, the {\em all-subgraphs centrality} of $v$ in $G$ is
	\[
	\allsub(v,G)\ =\ \log_2 |\A(v,G)|.
	\]

	\item[\textit{Closeness.}] This is a well-known measure introduced back in the 1960s~\cite{Sabi66}. It is usually called a geometrical measure since it relies on the distance inside a graph. It essentially states that the closer a vertex is to everyone in the graph the more central it is. Formally, given a graph $G$ and a vertex $v \in V(G)$, the {\em closeness centrality} of $v$ in $G$ is the ratio
	\[
	\closeness(v,G) = \frac{1}{\sum_{u \in K_v(G)} d_G(v,u)}.
	\]
	Let us clarify that we define the sum of distances inside a component of $G$ since the distance between two vertices in different components of $G$ is by definition infinite.
\end{description}

%\jorge{We will need notation to refer values of centrality inside a graph. Something like $A(G,C) = \{C(v,G)\mid v\in V(G)\}$.}

\section{Subgraph-based Centrality Measures}\label{sec:family}

As already discussed in the Introduction, a natural way of measuring the importance of a vertex in a graph is to count the relevant connected subgraphs surrounding it, and then apply a certain filtering function from the non-negative integers to the reals on top of the count. Of course, the relevant subgraphs and the adopted filtering function are determined by the intention of the centrality measure.
Interestingly, both the stress and the all-subgraphs centrality measures are actually based on this principle. Let us elaborate further on this. Consider a graph $G$ and a vertex $v \in V(G)$:

\begin{itemize}
	\item For the stress centrality, the important subgraphs for $v$ in $G$ are the shortest paths that go via $v$ in $G$, and the filtering function is $f_{\times 2}(x)=2x$ since each shortest path is counted twice.
	In other words, with $G_{\pi} $ being the  graph that corresponds to a path $\pi$,
	\[
	\stress(v,G) \ =\ f_{\times 2}\left(\left | \bigcup_{u,w \in V(G) \setminus \{v\}} \left\{G_\pi \mid \pi \in S_{G}^{v}(u,w)\right\} \right | \right).
	\]

	\item For the all-subgraphs centrality, the important subgraphs for $v$ in $G$ are the connected subgraphs of $G$ that contain $v$, that is, the set $\A(v,G)$, and the filtering function is $\log_2$.
	Indeed, by definition, we have that
	\[
	\allsub(v,G)\ =\ \log_2 |\A(v,G)|.
	\]
\end{itemize}

We proceed to formalize the above simple principle, originally introduced in~\cite{RiSa20}, which gives rise to a family of centrality measures, and then highlight our main research questions.

%The rest of the section is devoted to formalizing the above simple idea, which in turn gives rise to a family of centrality measures for undirected graphs based on subgraph motifs, and highlighting our main research questions.

\medskip

\noindent 
\paragraph{Subgraph-based Centrality Measures.} We first need a mechanism that allows us to specify what are the important subgraphs for a vertex $v$ in a graph $G$. This is done via the notion of {\em subgraph family}, which is defined as a function from vertex-graph pairs to sets of graphs that is closed under isomorphism, that is, a function $\F : \vg\ \rightarrow\ \mathcal{P}(\gr)$ such that:

\begin{itemize} 
	\item for every $(v,G) \in \vg$, $\F(v,G) \subseteq \A(v,G)$, that is, $\F$ assigns to each $(v,G) \in \vg$ a set of connected subgraphs of $G$ surrounding $v$, and
	
	\item for every two pairs $(v_1,G_1) \in \vg$ and $(v_2,G_2) \in \vg$ such that $(v_1,G_1) \simeq (v_2,G_2)$ witnessed by $h$, there exists a bijection $\mu : \F(v_1,G_1) \rightarrow \F(v_2,G_2)$ such that, for every $G' \in \F(v_1,G_1)$, $\mu(G') = (\{h(v) \mid v \in V(G')\},\{\{h(v),h(u)\} \mid (v,u) \in G'\})$.
\end{itemize}

\noindent We also need the notion of {\em filtering function}, which, as said above, is simply a function of the form $f : \mathbb{N} \rightarrow \mathbb{R}$.
We are now ready to define subgraph-based centrality measures:

\begin{definition}[\textbf{$\langle \F,f \rangle$-measure}]\label{def:subc}
	Consider a subgraph family $\F$ and a filtering function $f$. The {\em $\langle \F,f \rangle$-measure} is the function 
	$\subc{\F,f} : \vg\ \rightarrow\ \mathbb{R}$
	such that, for every pair $(v,G) \in \vg$, it holds that $\subc{\F,f}(v,G) = f(|\F(v,G)|)$.
	%
	%\begin{enumerate} 
	%	\item for every pair $(v,G) \in \vg$, $\subc{\F,f}(v,G) = f(|\F(v,G)|)$, and
	%	\item for every two pairs $(v_1,G_1) \in \vg$ and $(v_2,G_2) \in \vg$, $(v_1, G_1) \simeq (v_2,G_2)$ implies $\subc{\F,f}(v_1,G_1) = \subc{\F,f}(v_2,G_2)$. 
	%\end{enumerate}
	%
	%\noindent 
	\hfill\markfull
\end{definition}

Since, by definition, subgraph families are closed under isomorphism, it is straightforward to see that each $\langle \F,f \rangle$-measure defines a valid centrality measure.

\begin{lemma}
	For a subgraph family $\F$ and a filtering function $f$, it holds that the {\em $\langle \F,f \rangle$-measure} is a centrality measure. 
\end{lemma}

We say that a centrality measure $\Ce$ is a subgraph-based centrality measure if there are a subgraph family $\F$ and a filtering function $f$
such that $\Ce$ coincides with the $\langle \F,f \rangle$-measure, i.e., for every pair $(v,G) \in \vg$, $\Ce(v,G) = \subc{\F,f}(v,G)$.
Coming back to our discussion on stress and all-subgraph centralities, assuming that $\mathsf{S}$ is the subgraph family such that 
\[
\mathsf{S}(v,G)\ =\ \bigcup_{u,w \in V(G) \setminus \{v\}} \left\{G_\pi \mid \pi \in S_{G}^{v}(u,w)\right\},
\]
it is straightforward to verify that
\[
\stress\ =\ \subc{\mathsf{S},f_{\times 2}} \quad\text{and}\quad \allsub = \subc{\mathsf{A},\log_2}.
\]

\noindent 
\paragraph{Main Research Questions.} Having the family of subgraph-based centrality measures in place, the natural question that comes up concerns its absolute expressive power. In other words, we are interested in the following research question:

\medskip

\noindent \textit{\textbf{Question I: When is a centrality measure a subgraph-based centrality measure?}}

\medskip

One may wonder whether the above question is conceptually trivial in the sense that every centrality measure can be expressed as a subgraph-based centrality measure by choosing the subgraph family and the filtering function in the proper way as done for $\stress$ and $\allsub$.
It turns out that there are measures that are {\em not} subgraph-based.
% which renders the above question non-trivial.

\begin{proposition}\label{pro:no-subc-measure}
	There is a centrality measure that is not a subgraph-based measure.
\end{proposition}

\begin{proof}
	Consider the centrality measure $\Ce$ such that, for every $(v,G) \in \vg$, it holds that $\Ce(v,G)\ =\ |V(G)|$, i.e., it simply assigns to each vertex $v$ in a graph $G$ the number of vertices occurring in $G$. It suffices to show that $\Ce$ is not subgraph-based even if we focus on the set of graphs $\gr^\star$ consisting of $G_1 = (\{u_1\},\emptyset)$, $G_2 = (\{u_2,v_2\},\emptyset)$, and $G_3 = (\{u_3,v_3,w_3\},\emptyset)$.
	By contradiction, assume that $\Ce$ is a subgraph-based measure over $\gr^\star$. Thus, there exists a subgraph family $\F$ and a filtering function $f$ such that, for every $G \in \gr^\star$ and $v \in V(G)$, $\Ce(v,G) = \subc{\F,f}(v,G)$. We observe that:
	
		\begin{enumerate}
			\item For every $(v,G) \in \ins{V} \times \gr^\star$ with $v \in V(G)$, it holds that $\subc{\F,f}(v,G) \in \{1,2,3\}$, i.e., we have three distinct values. This follows by the definition of $\Ce = \subc{\F,f}$.
			
			\item For every $(v,G) \in \ins{V} \times \gr^\star$, it holds that $|\F(v,G)| \in \{0,1\}$, i.e., we have two possible sizes for the sets of connected subgraphs.
		\end{enumerate}
		
		\noindent Now, by the pigeonhole principle, we can safely conclude that there are two distinct pairs $(v,G),(u,G') \in  \ins{V} \times \gr^\star$ with $|\F(v,G)| = |\F(u,G')|$ such that $\subc{\F,f}(v,G) \neq \subc{\F,f}(u,G')$. But this contradicts the fact that $f$ is a function, and the claim follows.
\end{proof}

As we shall see, not only artificial measures as the one employed in the proof of Proposition~\ref{pro:no-subc-measure}, but also well-known centrality measures from the literature (such as $\closeness$) are {\em not} subgraph-based. We are going to prove such inexpressibility results by using the technical tools developed towards answering Question I.

In several applications that involve graphs, we are more interested in the relative than the absolute importance of a vertex in a graph. More precisely, we are interested in the ranking of the vertices of a graph induced by a measure $\Ce$, and not in the absolute value assigned to a vertex by $\Ce$. This brings us to the next technical notion:

\begin{definition}[\textbf{Induced Ranking}]\label{def:ranking}
	Let $\Ce$ be a centrality measure. The {\em ranking induced by $\Ce$}, denoted $\rank(\Ce)$, is the binary relation
	\[
	\left\{((u,G),(v,G)) \mid u,v \in V(G) \text{ and } \Ce(u,G) \leq \Ce(v,G)\right\}
	\]		
	over $\vg$.
	$\Ce$ is a {\em subgraph-based centrality measure relative to the induced ranking} if there are a subgraph family $\F$ and a filtering function $f$ with $\rank(\Ce) = \rank(\subc{\F,f})$. \hfill\markfull
\end{definition}

Interestingly, although the measure employed in the proof of Proposition~\ref{pro:no-subc-measure} is not subgraph-based, it is easy to show that it is a subgraph-based measure relative to the induced ranking.
In particular, by defining the subgraph family $\F$ as $\F(v,G) = \{G_v\}$, for every $(v,G) \in \vg$, and the filtering function as the identity, it is not difficult to see that $\rank(\Ce) = \rank(\subc{\F,f})$.
\begin{comment}
{\color{red} In particular, with $\hat{G}$ being the graph $(\{v_1,v_2,v_3\},\{\{v_1,v_2\}\})$ used in the proof of Proposition~\ref{pro:no-subc-measure}, by defining the subgraph family $\F$ as
\[
\F(v, G)\ =\ \begin{cases}
\emptyset & G \neq \hat{G}\\
\emptyset & G = \hat{G} \text{ and } v = v_1\\
\{G_{v_2}, \hat{G}[K_{v_2}(\hat{G})]\} & G = \hat{G} \text{ and } v = v_2\\
\{G_v\} & G = \hat{G} \text{ and } v = v_3,
\end{cases}
\]
and the filtering function $f$ as $f(0)=1$, $f(1) = 3$ and $f(2) = 2$, it is not difficult to see that $\rank(\Ce) = \rank(\subc{\F,f})$.}
\end{comment}
This observation brings us to our next research question:

\medskip

\noindent \textit{\textbf{Question II: When is a centrality measure a subgraph-based centrality measure relative to the induced ranking?}}

\medskip

As we shall see, the above question is conceptually non-trivial, i.e.,  there are measures that are {\em not} subgraph-based measures relative to the induced ranking. In particular, we will see that there are well-established measures (such as $\closeness$) that are {\em not} subgraph-based centrality measures relative to the induced ranking. Such inexpressibility results are shown by exploiting the tools developed towards answering Question II.
\section{Characterizing Subgraph-based Centrality Measures}\label{sec:characterization}

We proceed to provide an answer to Question I. More precisely, our goal is to isolate a structural property $P$ over centrality measures that precisely characterizes subgraph-based measures, that is, for an arbitrary measure $\Ce$, $\Ce$ is a subgraph-based measure iff $\Ce$ enjoys $P$.
Interestingly, the desired property can be somehow extracted from the proof of Proposition~\ref{pro:no-subc-measure}. The crucial intuition provided by that proof is that the absolute expressiveness of subgraph-based measures is tightly related to the amount of connected subgraphs that are available for assigning different centrality values to vertices. In other words, a measure that assigns ``too many'' values among vertices that are surrounded by ``too few'' connected subgraphs cannot be expressed as a subgraph-based measure. We proceed to formalize this intuition.

We first collect all the different values assigned by a centrality measure $\Ce$ to the vertices of a graph $G$ that are surrounded by a bounded number of connected subgraphs of $G$. In particular, for $n>0$, we define the set of real values
\[
\val{G}{n}(\Ce)\ =\ \left\{\Ce(v,G) \mid v \in V(G) \text{ and } |\A(v,G)| \leq n\right\}.
\]
We can then easily collect all the values assigned by $\Ce$ to the vertices of $\ins{V}$ that are surrounded by a bounded number of connected subgraphs in some graph. In particular, for $n>0$,
\[
\val{}{n}(\Ce)\ =\ \bigcup_{G \in \gr} \val{G}{n}(\Ce).
\]
We now define the following property over centrality measures:

\begin{definition}[\textbf{Bounded Value Property}]\label{def:bvp}
	A measure $\Ce$ enjoys the {\em bounded value property} if, for every $n>0$, $|\val{}{n}(\Ce)| \leq n+1$. \hfill\markfull
\end{definition}

The bounded value property captures the key intuition discussed above. It actually bounds the number of different values that can be assigned among vertices that are surrounded by a limited number of connected subgraphs; hence the name ``bounded value property''.
Observe that the measure $\Ce$ devised in the proof of Proposition~\ref{pro:no-subc-measure} does not enjoy the bounded value property; indeed, $|\val{}{1}(\Ce)| \geq 3 > 2$.
Interestingly, the bounded value property is all we need towards a precise characterization of subgraph-based measures.

\begin{theorem}\label{the:characterization}
	Consider a centrality measure $\Ce$. The following statements are equivalent:
	
	\begin{enumerate}
		\item $\Ce$ is a subgraph-based centrality measure.
		
		\item $\Ce$ enjoys the bounded value property.
	\end{enumerate}
\end{theorem}

\begin{proof}
$(1 \Rightarrow 2)$  By contradiction, assume that $\Ce$ does not enjoy the bounded value property, namely there exists an integer $n\geq1$ such that $|\val{}{n}(\Ce)| > n+1$. 
By hypothesis, $\Ce$ is a subgraph-based centrality measure, and thus, there exist a subgraph family $\F$ and a filtering function $f$ such that the following holds: for every $(v,G)\in \vg$, $\Ce(v,G) = \subc{\F,f}(v,G)$. 
We now define the set 
\[
B_n\ =\ \left\{|\F(v,G)|\mid (v,G)\in \vg \text{ and } |\A(v,G)| \leq n\right\}. 
\]
Clearly, $|B_n| \leq n+1$ since $\F(v,G) \subseteq \A(v,G)$. 
Let $h : \val{}{n}(\Ce) \to B_n$ be such that
\[
h(\Ce(v,G))\ =\ |\F(v,G)|.
\]
By the pigeonhole principle, $h$ is not injective, i.e., there exist $\Ce(v_1,G_1)$ and $\Ce(v_2,G_2)$ such that $\Ce(v_1,G_1) \neq \Ce(v_2,G_2)$ but $|\F(v_1,G_1)| = |\F(v_2,G_2)|$. This contradicts the fact that $\Ce(v_1,G_1) = f(|\F(v_1,G_1)|) \neq f(|\F(v_2,G_2)|) = \Ce(v_2,G_2)$, and the claim follows.

\medskip

$(2 \Rightarrow 1)$ The goal is to show that there exist a subgraph family $\F$ and a filtering function $f$ such that, for every $(v,G) \in \vg$, $\Ce(v,G) = \subc{\F,f}(v,G)$.
We start by defining a total order $\preceq_{\Ce}$ over the set of values $\val{}{}(\Ce) = \bigcup_{i=1}^{\infty}\val{}{i}(\Ce)$. 
By definition, for every $n,m > 0$ such that $n \leq m$, it holds that $\val{}{n}(\Ce)\subseteq \val{}{m}(\Ce)$. In other words, as we increase the integer $n$ we are adding new values to the set $\val{}{n}(\Ce)$.
We can now define the binary relation $\preceq_{\Ce}$ over $\val{}{}(\Ce)$ as follows: for each $a,b \in \val{}{}(\Ce)$, if there exists $n$ such that $a \in \val{}{n}(\Ce)$ but $b \not\in \val{}{n}(\Ce)$ then $a \preceq_{\Ce} b$, if not, then $a \preceq_{\Ce} b$ if $a\leq b$. 
It is easy to see that $\preceq_{\Ce}$ is a total order over $\val{}{}(\Ce)$, and thus, it is a total order over $\val{}{n}(\Ce)$ for each $ n> 0$. For notational convenience, in the rest of the proof we assume that  $\val{}{}(\Ce) = \{a_1,a_2,a_3,\ldots\}$ and $a_1 \preceq_{\Ce} a_2 \preceq_{\Ce} a_3 \preceq_{\Ce} \cdots$.

By exploiting the total order $\preceq_{\Ce}$ over $\val{}{}(\Ce)$, we proceed to define a subgraph family $\F$. Consider an arbitrary pair $(v,G) \in \vg$, and let $n = |\A(v,G)|$. 
By hypothesis, $\Ce$ enjoys the bounded value property, which in turn implies that $|\val{}{n}(\Ce)| \leq n+1$. Therefore, $\Ce(v,G)$, which belongs to $\{a_1,a_2,...,a_{|\val{}{n}(\Ce)|}\}$, is equal to $\val{}{n}(\Ce)$. %Here we ordered the set $\val{}{n}(\Ce)$ by the relation $\preceq_{\Ce}$. 
We further observe that $\A(v,G)$ is a finite set, and we let $\A(v,G) = \{S_1,S_2,\ldots,S_n\}$. 
Here we assume an arbitrary order for $\A(v,G)$ that has the following property: for every pair $(v',G')$ with $(v,G)\simeq (v',G')$, assuming that $\A(v',G') = \{S'_1,S'_2,\ldots,S'_n\}$, it holds that $(v,S_i)\simeq (v'_i,S'_i)$ for every $i\in\{1,\ldots,n\}$.
The subgraph family $\F$ is defined as follows:
\[
\Ce(v,G) = a_i \quad \text{implies} \quad \F(v,G) = \{S_1,...,S_{i-1}\}.
\]
This is indeed a subgraph family since $\F(v,G) \subseteq \A(v,G)$, while the chosen order for $\A(v,G)$ and the fact that $\Ce$ is (by definition) closed under isomorphism ensures closure under isomorphism. Notice that $|\F(v,G)| = i-1$ for $i \in \{1,\ldots,|\val{}{n}(\Ce)|+1\}$.
Finally, we define the filtering function $f : \mathbb{N} \to \val{}{}(\Ce)$ as follows: for each $i \in \mathbb{N}$,
\[
f(i)\ =\ a_{i+1}. 
\]

We proceed to show that $\F$ and $f$ capture our intention, that is, for every $(v,G) \in \vg$, $\Ce(v,G) = \subc{\F,f}(v,G)$, which will establish Theorem~\ref{the:characterization} .
Let $n = |\A(v,G)|$. If $\Ce(v,G) = a_i\in \val{}{n}(\Ce)$, then $|\F(v,G)| = i-1$. Therefore, $f(|\F(v,G)|) = \subc{\F,f}(v,G) = \Ce(v,G)$.
Conversely, if $\subc{\F,f}(v,G) = a_i$, then $|\F(v,G)| = i-1$, and thus, by construction, $\Ce(v,G) = a_i$.
\end{proof}

The above characterization, apart from giving a definitive answer to Question I, it provides a useful tool for establishing inexpressibility results. To show that a centrality measure $\Ce$ is not a subgraph-based measure it suffices to show that there exists an integer $n>0$ such that $|\val{}{n}(\Ce)| > n+1$. For example, we can show that 
%in the case of the closeness centrality measure, we can show that 
$|\val{}{5}(\closeness)| > 6$, and therefore:
%even if we concentrate on connected graphs, 
%which in turn implies that:

\begin{proposition}\label{pro:closeness}
	$\closeness$ is not a subgraph-based measure.
	%even if we focus on connected graphs.
\end{proposition}

Without Theorem~\ref{the:characterization} in place, it is completely unclear how one can prove that $\closeness$ (or any other established measure) is not a subgraph-based measure. More inexpressibility results concerning well-established centrality measures are discussed in Section~\ref{sec:classification}.
\section{Characterizing Subgraph-based Measures Relative to the Induced Ranking}\label{sec:characterization-ranking}

We now focus on Question II. Our goal is to isolate a structural property $P$ over centrality measures that precisely characterizes subgraph-based measures relative to the induced ranking, i.e., for an arbitrary measure $\Ce$, $\Ce$ is subgraph-based relative to the induced ranking iff $\Ce$ enjoys the property $P$.
%
%Despite our efforts, we have not managed to define $P$ in the spirit of the bounded value property (see Definition~\ref{def:bvp}).
It turns out that $P$ can be defined by exploiting a certain notion of graph coloring relative to a centrality measure.
%
%At the end of this section, we further discuss how the coloring-based property helps to obtain a bounded-value-like property that is a necessary condition for a measure being subgraph motif relative to the induced ranking.
%
%As we shall see in Section~\ref{sec:classification}, the latter provides a convenient tool for showing that several existing measures are not subgraph motif relative to the induced ranking.

%It should be clear that the bounded value property introduced in the previous section (see Definition~\ref{def:bvp}) cannot be the property $P$ that we are looking for. This is because, as discussed in Section~\ref{sec:family}, there is a measure (see the one devised in the proof of Proposition~\ref{pro:no-subc-measure}) that is not subgraph motif, and thus, by Theorem~\ref{the:characterization}, does not enjoy the bounded value property, but it is subgraph motif relative to the induced ranking. 
%
%It turns out that the property $P$ can be defined by exploiting a certain notion of graph coloring relative to a measure.

\medskip

\noindent \paragraph{Graph Colorings.} The high-level idea is to consider the sizes of the available subgraph families that can be assigned to a vertex $v$ in a graph $G$, i.e., the set of integers $\{0,\ldots,|\A(v,G)|\}$, as available colors. We can then refer to a precoloring of $\vg$ (i.e., of all the possible graphs) as a function $\pc : \vg \rightarrow \mathbb{N}$ that assigns to each vertex $v$ in a graph $G$ only available colors from $\{0,\ldots,|\A(v,G)|\}$.
Then, the goal is to isolate certain properties of such a precoloring of $\vg$ that leads to the desired characterization, i.e., a measure $\Ce$ is subgraph-based relative to the induced ranking iff there exists a precoloring of $\vg$ that enjoys the properties in question.
Such a characterization tells us that for a centrality measure being subgraph-based relative to the induced ranking is tantamount to the fact that there are enough colors (i.e., sizes of sugbraph families, but {\em without} considering their actual topological structure) that allow us to color $\vg$ in a valid way, namely in a way that the crucial properties are satisfied.
%
%Let us remark that the coloring-based characterization discussed above is different in spirit compared to the characterization provided by Theorem~\ref{the:characterization} based on the bounded value property. The question whether subgraph motif measures relative to the induced ranking can be characterized via some kind of bounded value property is discussed at the end of the section. 
%
We proceed to formalize the above discussion about colorings.

Given a set $S \subseteq \vg$, a {\em precoloring of $S$} is a function $\pc : S \rightarrow \mathbb{N}$ such that, for every $(v,G)\in S$, $\pc(v,G) \in \{0,\ldots,|\A(v,G)|\}$.
The first key property of such a precoloring states that the values assigned by a measure $\Ce$ to the vertices of a graph $G$ should be respected, i.e., vertices with different centrality values get different colors. This is formalized as follows:
%leads to the following property over precolorings: 

\begin{definition}[\textbf{Non-Uniform $\Ce$-Injectivity}]\label{def:non-uniform-injective}
	Consider a set $S \subseteq \vg$, and a precoloring $\pc : S \rightarrow \mathbb{N}$ of $S$. Given a centrality measure $\Ce$, we say that $\pc$ is {\em non-uniformly $\Ce$-injective} if, for every $(u,G),(v,G) \in S$, $\Ce(u,G) \neq \Ce(v,G)$ implies $\pc(u,G) \neq \pc(v,G)$. \hfill\markfull
\end{definition}

The term non-uniform in the above definition refers to the fact that $\Ce$-injectivity is only enforced inside a certain graph, and not across all the graphs mentioned in $S$, i.e., it might be the case that a non-uniformly $\Ce$-injective precoloring of $S$ assigns to $(u,G),(v,G')$, where $G\neq G'$ and $\Ce(u,G) \neq \Ce(v,G')$, the same color.

The second key property of a precoloring $S$ states that $S$ should be consistent with the induced ranking, not only inside a certain graph, but also among different graphs mentioned in $S$. In other words, if $(u,G)$ comes before $(v,G)$ and $(u',G')$ comes before $(v',G')$, then one of the following should hold: $(u,G)$ and $(v',G')$ get different colors, or $(u',G)$ and $(v,G')$ get different colors. This is formalized as follows:
%This leads to the following property: 

\begin{definition}[\textbf{$\Ce$-Consistency}]\label{def:rank-consistent}
	Consider a set $S \subseteq \vg$, and a precoloring $\pc : S \rightarrow \mathbb{N}$ of $S$. Given a measure $\Ce$, we say that $\pc$ is {\em $\Ce$-consistent} if, for every $(u,G),(v,G),(u',G'),(v',G') \in S$, the following holds: if $\Ce(u,G) < \Ce(v,G)$ and $\Ce(u',G') < \Ce(v',G')$, then $\pc(u,G) \neq \pc(v',G')$ or $\pc(u',G) \neq \pc(v,G')$. \hfill\markfull
\end{definition}

Putting together the above two properties over precolorings, we get the notion of $\Ce$-colorability of a set $S \subseteq \vg$:

\begin{definition}[\textbf{$\Ce$-Colorability}]\label{def:coloring}
	We say that a set $S \subseteq \vg$ is {\em $\Ce$-colorable}, for some measure $\Ce$, if there exists a precoloring of $S$ that is non-uniformly $\Ce$-injective and $\Ce$-consistent. \hfill\markfull
\end{definition}

%\medskip

\noindent \paragraph{The Characterization.}
Interestingly, $\Ce$-colorability is all we need towards the desired characterization, namely a measure $\Ce$ is subgraph-based relative to the induced ranking iff $\vg$ (i.e., all possible graphs) is $\Ce$-colorable.
We further show that the $\Ce$-colorability of $\vg$ is equivalent to the $\Ce$-colorability of every finite set $S \subsetneq \vg$. The latter, apart from being interesting in its own right, it provides a tool that is more convenient than the $\Ce$-colorability of $\vg$ for classifying measures as subgraph-based relative to the induced ranking.

\begin{comment}
be colored in a valid way, that is, non-uniform $\Ce$-injectivity and $\Ce$-consistency are fulfilled.
%
For brevity, we say that $\vg$ is {\em $\Ce$-colorable} if there exists a precoloring of $\vg$ that is uniformly $\Ce$-injective and $\Ce$-consistent. 
%
In addition, we show that $\vg$ is $\Ce$-colorable iff is {\em finitely} $\Ce$-colorable, i.e., every finite

Note that, for proving this equivalence, we heavily rely on an intermediate statement, which is explicitly given below in the characterization, that states that $\vg$ is finitely colorable, i.e., every {\em finite} subset of $\vg$ can be colored in a valid way.
\end{comment}

\begin{theorem}\label{the:characterization-ranking}
	Consider a centrality measure $\Ce$. The following statements are equivalent:
	
	\begin{enumerate}
		\item $\Ce$ is a subgraph-based centrality measure relative to the induced ranking.
		\item Every finite set $S \subsetneq \vg$ is $\Ce$-colorable.
		\item $\vg$ is $\Ce$-colorable.
	\end{enumerate}
\end{theorem}

\begin{comment}
\begin{theorem}\label{the:characterization-ranking}
	Consider a centrality measure $\Ce$. The following statements are equivalent:
	\begin{enumerate}
		\item $\Ce$ is a subgraph motif measure relative to the induced ranking.
		\item For every finite set $S \subsetneq \vg$, there exists a precoloring of $S$ that is uniformly $\Ce$-injective and $\Ce$-consistent.
		\item There exists a precoloring of $\vg$ that is uniformly $\Ce$-injective and $\Ce$-consistent.
	\end{enumerate}
\end{theorem}
\end{comment}

To show the above characterization, it suffices to establish the sequence of implications $(1) \Rightarrow (2) \Rightarrow (3) \Rightarrow (1)$.
The implication $(1) \Rightarrow (2)$ is a rather easy one and its full proof is given below. The proofs of the implications $(2) \Rightarrow (3)$ and $(3) \Rightarrow (1)$ are more interesting and we discuss their key ingredients below.

\medskip

\noindent \underline{Implication $(1) \Rightarrow (2)$}

\smallskip

\noindent Since, by hypothesis, $\Ce$ is a subgraph-based measure relative to the induced ranking, there are a subgraph family $\F$ and a filtering function $f$ such that $\rank(\Ce) = \rank(\subc{\F,f})$. Given a finite set $S \subsetneq \vg$, we define the function $\pc_S : S \rightarrow \mathbb{N}$ as follows: for every $(v,G) \in S$, $\pc_S(v,G) = |\F(v,G)|$. It is clear that $\pc_S$ is a precoloring of $S$ since, by definition, $\F(v,G) \subseteq \A(v,G)$, and thus, $\pc_S(v,G) \in \{0,\ldots,|\A(v,G)|\}$. It remains to show that $\pc_S$ is non-uniformly $\Ce$-injective and $\Ce$-consistent, which in turn implies that $S$ is $\Ce$-colorable:

\begin{description}
	\item[Non-uniformly $\Ce$-injective.] Since $\rank(\Ce) = \rank(\Ce\langle\F,f\rangle)$, for every $(u_1, G),(u_2,G) \in S$, it holds that $\Ce(u_1, G) \neq \Ce(u_2,G)$ iff $\subc{\F,f}(u_1, G)\not=\subc{\F,f}(u_2, G)$. Therefore, $\pc_S(u_1,G) = |\F(u_1,G)| \neq |\F(u_2,G)| = \pc_S(u_2,G)$, and the claim follows. 
	%which implies that $\pc_S$ is non-uniformly $\Ce$-injective.
	
	\item[$\Ce$-consistent.] By contradiction, assume that there are $(v_1,G_1),(v_2,G_1),(u_1,G_2)$ and $(u_2,G_2)$ such that $\Ce(v_1,G_1)<\Ce(v_2,G_1)$ and $\Ce(u_1,G_2) >\Ce(u_2,G_2)$ but $\pc_S(u_1,G_2) = \pc_S(v_1,G_1)$ and $\pc_S(v_2,G_1) = \pc_S(u_2,G_2)$. Therefore, $|\F(v_1,G_1)| = |\F(u_1,G_2)|$ and $|\F(v_2,G_1)| = |\F(u_2,G_2)|$. Consequently, using the fact that $\rank(\Ce) = \rank(\Ce\langle\F,f\rangle)$, $\subc{\F,f}(v_1,G_1) < \subc{\F,f}(v_2,G_1) = \subc{\F,f}(u_2,G_2)<\subc{\F,f}(u_1,G_2)=\subc{\F,f}(v_1,G_1)$, which is clearly a contradiction, and the claim follows.
\end{description}

\begin{comment}
\medskip

\noindent \underline{Implication $(1) \Rightarrow (2)$}

\smallskip

\noindent This implication is a rather easy one to prove. Since, by hypothesis, $\Ce$ is a subgraph motif measure relative to the induced ranking, there are a subgraph family $\F$ and a filtering function $f$ such that $\rank(\Ce) = \rank(\subc{\F,f})$. Given a finite set $S \subsetneq \vg$, we define the function $\pc_S : S \rightarrow \mathbb{N}$ as follows: for every $(v,G) \in S$, $\pc_S(v,G) = |\F(v,G)|$. It is clear that $\pc_S$ is a precoloring of $S$ since, by definition, $\F(v,G) \subseteq \A(v,G)$, and thus, $\pc_S(v,G) \in \{0,\ldots,|\A(v,G)|\}$. It is also not difficult to show that $\pc_S$ is non-uniformly $\Ce$-injective and $\Ce$-consistent, which in turn implies that $S$ is $\Ce$-colorable.
\end{comment}

\medskip

\noindent \underline{Implication $(2) \Rightarrow (3)$}

\smallskip

\noindent The proof of this implication heavily relies on an old result that goes back in 1949 by Rado~\cite{rado49} known as {\em Rado's Selection Principle}.
We write $\mathcal{P}_{\mathsf{fin}}(A)$ for the finite powerset of a set $A$, i.e., the set that collects all the {\em finite} subsets of $A$. Furthermore, given a function $f : A \rightarrow B$, we write $f_{|C}$ for the restriction of $f$ to $C \subseteq A$.

\begin{theorem}[Rado's Selection Principle]\label{the:rados}
	Let $A$ and $B$ be arbitrary sets. Assume that, for each $C \in \mathcal{P}_{\mathsf{fin}}(A)$, $f_C$ is a function $C \rightarrow B$ (a so-called ``local function''). Assume further that, for every $x \in A$, the set $\{f_C(x) \mid C \in \mathcal{P}_{\mathsf{fin}}(A) \text{ and } x \in C\}$ is finite. Then, there is a function $f : A \rightarrow B$ (a so-called ``global function'') such that, for every $C \in \mathcal{P}_{\mathsf{fin}}(A)$, there is $D \in \mathcal{P}_{\mathsf{fin}}(A)$ with $C \subsetneq D$ and $f_{|C} = {f_{D}}_{|C}$.
\end{theorem}

Several proofs and applications of Rado's Theorem can be found in~\cite{mirsky71}. We proceed to discuss how it is used to prove $(2) \Rightarrow (3)$.
By hypothesis, for each $S \in \mathcal{P}_{\mathsf{fin}}(\vg)$, there exists a precoloring of $S$, i.e., a function $\pc_S : S \rightarrow \mathbb{N}$ that is non-uniformly $\Ce$-injective and $\Ce$-consistent. Since, for every $(v,G) \in \vg$, $\A(v,G)$ is finite, we can conclude that the following holds: for every $(v,G) \in \vg$, the set $\{\pc_S(v,G) \mid S \in \mathcal{P}_{\mathsf{fin}}(\vg) \text{ and } (v,G) \in S\}$ is finite. This allows us to apply Theorem~\ref{the:rados} with $A = \vg$ and $B = \mathbb{N}$. Therefore, there exists a function $f : \vg \rightarrow \mathbb{N}$ such that, for every $S \in \mathcal{P}_{\mathsf{fin}}(\vg)$, there exists $S' \in \mathcal{P}_{\mathsf{fin}}(\vg)$ with $S \subsetneq S'$ and $f_{|S} = {\pc_{S'}}_{|S}$. Interestingly, by exploiting the latter property of the function $f$ guaranteed by Theorem~\ref{the:rados}, and the fact that, for each $S \in \mathcal{P}_{\mathsf{fin}}(\vg)$, $\pc_S$ is a precoloring of $S$ that is non-uniformly $\Ce$-injective and $\Ce$-consistent, it is not difficult to show that $f$ is a precoloring of $\vg$ that is non-uniformly $\Ce$-injective and $\Ce$-consistent, and item (3) follows.

\medskip

\noindent \underline{Implication $(3) \Rightarrow (1)$}

\smallskip

\noindent We finally discuss the proof of the last implication. The goal is to devise a subgraph family $\F$ and a filtering function $f$ such that $\rank(\Ce) = \rank(\subc{\F,f})$, which in turn proves item $(1)$. By hypothesis, there exists a precoloring $\pc$ of $\vg$ that is non-uniformly $\Ce$-injective and $\Ce$-consistent. We define $\F$ in such way that, for every $(v,G) \in \vg$, $|\F(v,G)| = \pc(v,G)$; note that such a subgraph family exists since $\pc(v,G) \in \{0,\ldots,|\A(v,G)|\}$.
Now, defining the filtering function $f$ is a non-trivial task. Let $R_{\pc}$ be the relation
\begin{multline*}
\left\{(i,j) \in \mathbb{N} \times \mathbb{N} \mid \text{ there are } (u,G), (v,G) \text{ in } \vg \text{ such that}\right.\\ \left.\Ce(u,G) < \Ce(v,G), \pc(u,G) = i, \text{ and } \pc(v,G) = j\right\}.
\end{multline*}
The fact that $\pc$ is non-uniformly $\Ce$-injective allows us to conclude that $R_\pc$ is irreflexive. Moreover, the $\Ce$-consistency of $\pc$ implies that $R_\pc$ is asymmetric. Observe now that if we extend $R_\pc$ into a total order $R_{\pc}^{\star}$ over $\mathbb{N}$, and then show that $R_{\pc}^{\star}$ can be embedded into a carefully chosen countable subset $N$ of $\mathbb{R}$, then we obtain the desired filtering function $f$, which assigns real numbers to the sizes of the subgraph families assigned to the pairs of $\vg$ by $\F$ as dictated by the embedding of $R_{\pc}^{\star}$ into $N \subsetneq \mathbb{R}$.
Let us now briefly discuss how this is done. The binary relation $R_\pc$ is first extended into the strict partial order $R_{\pc}^{+}$ by simply taking its transitive closure. Now, the fact that $R_{\pc}^{+}$ can be extended into a total order $R_{\pc}^{\star}$ over $\mathbb{N}$ follows by the {\em order-extension principle} (a.k.a.~{\em Szpilrajn Extension Theorem}), shown by Szpilrajn in 1930~\cite{szpi30}, which essentially states that every partial order can be extended into a total order.
Finally, the fact that $R_{\pc}^{\star}$ can be embedded into $N \subsetneq \mathbb{R}$ is shown via the {\em back-and-forth method}, a technique for showing isomorphism between countably infinite structures satisfying certain conditions.

\medskip

\noindent \paragraph{A Bounded-Value-Like Property.} An interesting question is whether we can isolate a property in the spirit of the bounded value property (see Definition~\ref{def:bvp}) that can characterize subgraph-based measures relative to the induced ranking. Despite our efforts, we have not managed to provide an answer to this question. On the other hand, we succeeded in isolating a bounded-value-like property that is a necessary condition for a measure being subgraph-based relative to the induced ranking.
It is clear that the bounded value property is not enough towards a necessary condition since, as discussed in Section~\ref{sec:family}, there is a measure (see the one devised in the proof of Proposition~\ref{pro:no-subc-measure}) that is not subgraph-based, which means that it does not enjoy the bounded value property, but it is subgraph-based relative to the induced ranking. 
On the other hand, to our surprise, a {\em non-uniform} version of the bounded value property leads to the desired necessary condition. Let us make this more precise.
The ranking induced by a measure $\Ce$ compares only the values of vertices of the same graph; a pair $((u,G),(v,G'))$, where $G \neq G'$, will never appear in $\rank(\Ce)$.
This led us to conjecture that for characterizing subgraph-based measures relative to the induced ranking, it suffices to bound the number of different values that can be assigned among vertices {\em inside the same graph} that are surrounded by a limited number of connected subgraphs.
This leads to the non-uniform version of the bounded value property:

\begin{definition}[\textbf{Non-Uniform Bounded Value Property}]\label{def:non-uniform-bvp}
	A measure $\Ce$ enjoys the {\em non-uniform bounded value property} if, for every $n>0$ and $G \in \gr$, $|\val{G}{n}(\Ce)| \leq n+1$. \hfill\markfull
\end{definition}

We can then show the following implication:

\begin{proposition}\label{pro:coloring-to-non-uniform-bvp}
	Consider a centrality measure $\Ce$. If there exists a precoloring of $\vg$ that is non-uniformly $\Ce$-injective, then $\Ce$ enjoys the non-uniform bounded value property.
\end{proposition}

\begin{proof}
	Let $pc$ be the non-uniform $\Ce$-injective precoloring of $\vg$, which exists by hypothesis.
	Consider an arbitrary graph $G$ and an integer $n>0$. We define the set
	\[
	S_n\ =\ \{pc(v,G)\mid |\A(v,G)|\leq n\}. 
	\]
	In simple words, $S_n$ collects all the colors assigned by $pc$ to vertices with at most $n$ connected subgraphs surrounding them. We then have that $|\val{G}{n}(\Ce)|\leq |S_n|$ since $pc$ is non-uniformly $\Ce$-injective. 
	Since $pc$ is a precoloring, $S_n \subseteq \{0,\ldots,n\}$, and thus, $|S_n|\leq n+1$. This in turn implies that $|\val{G}{n}(\Ce)|\leq n+1$, and the claim follows.
\end{proof}

 By combining Theorem~\ref{the:characterization-ranking} and Proposition~\ref{pro:coloring-to-non-uniform-bvp}, we get the following corollary, which states that the non-uniform bounded value property leads to the desired necessary condition:

 \begin{corollary}\label{cor:necessary-condition-ranking}
 	If a centrality measure is a subgraph-based measure relative to the induced ranking, then it enjoys the non-uniform bounded value property.
 \end{corollary}

The question whether the non-uniform bounded value property is also a sufficient condition is negatively settled by the next result:

\begin{proposition}\label{pro:non-uniform-bvp-not-sufficient}
	There exists a centrality measure that is not a subgraph-based measure relative to the induced ranking, but it enjoys the non-uniform bounded value property.
\end{proposition}

Let us stress that Corollary~\ref{cor:necessary-condition-ranking} equips us with a convenient tool for showing that a measure $\Ce$ is not a subgraph-based measure relative to the induced ranking: it suffices to show that there is $n>0$ and a graph $G$ such that $|\val{G}{n}(\Ce)| > n+1$. In the case of  closeness, we can show that there exists a graph $G$ such that $|\val{G}{5}(\closeness)| > 6$, which in turn implies that:

\begin{proposition}\label{pro:closeness-ranking}
	$\closeness$ is not a subgraph-based measure relative to the induced ranking.
\end{proposition}

More inexpressibility results of the above form concerning established centrality measures are presented and discussed in Section~\ref{sec:classification}.

\medskip

\noindent
\paragraph{Connected Graphs.} 
%After a quick inspection of 
The proof of Proposition~\ref{pro:closeness} establishes that 
%one can conclude that 
$\closeness$ is not a subgraph-based measure even if we concentrate on connected graphs.
On the other hand, the proof of Proposition~\ref{pro:closeness-ranking} heavily relies on the fact that the employed graphs are not connected.
This observation led us ask ourselves whether $\closeness$ is a subgraph-based measure relative to the induced ranking if we consider only connected graphs.
It turned out that, for connected graphs, not only $\closeness$, but {\em every} measure is subgraph-based relative to the induced ranking. We proceed to formalize this discussion.

Let $\vcg = \{(v,G) \in \vg \mid G \text{ is connected}\}$. For an arbitrary centrality measure $\Ce$, its version that operates only on connected graphs is defined as the function $\conC : \vcg \rightarrow \mathbb{R}$ such that, for every $(v,G) \in \vcg$, $\Ce(v,G) = \conC(v,G)$, i.e., it is the restriction of $\Ce$ over $\vcg$.
We then say that $\conC$ is a subgraph-based measure (resp., subgraph-based measure relative to the induced ranking) if there exist a subgraph family $\F$ and a filtering function $f$ such that $\conC = \con\subc{\F,f}$ (resp., $\rank(\Ce) \cap \vcg^2 = \rank(\subc{\F,f}) \cap \vcg^2$).
We can then establish the following result:

\begin{theorem}\label{the:ranking-connected-graphs}
	Consider a centrality measure $\Ce$. It holds that $\conC$ is a subgraph-based measure relative to the induced ranking.
\end{theorem}

\begin{proof}
We are going to define a subgraph family $\F$ and a filtering function $f$ such that $\rank(\Ce) \cap \vcg^2 = \rank(\subc{\F,f}) \cap \vcg^2$, which in turn implies that $\conC$ is a subgraph-based measure relative to the induced ranking, as needed.
Consider an arbitrary connected graph $G$. We first observe that, for every $v \in V(G)$, it holds that $|\A(v, G)|\geq |V(G)|$ since every path from $v$ to any other vertex in $G$ is a connected subgraph containing $v$. 
We then define the equivalence relation $\equiv_G$ over $V(G)$ as follows:  $v \equiv_G u$ if $\Ce(v,G) = \Ce(u,G)$. Let $V(G)/_{\equiv_G} = \{C_1,\ldots,C_m\}$ be the equivalence classes of $\equiv_G$. We can assume, without loss of generality, that, for every $i,j \in [m]$, with $C_i = [v]_{\equiv_G}$ and $C_j = [u]_{\equiv_G}$, $i<j$ implies $\Ce(v,G)< \Ce(u, G)$.
We then define the subgraph family $\F$ in such a way that, for every vertex $v \in V(G)$, $|\F(v,G)|= i-1$ if $[v]_{\equiv_G} = C_i$.\footnote{Note that for pairs $(u,G')$, where $G'$ is a non-connected graph, we can simply define $\F(u,G')$ as the empty set since it is irrelevant what $\F$ does over non-connected graphs.} Note that such a subgraph family $\F$ always exists since, as discussed above, $|\A(v,G)|\geq |V(G)|$, but we have that $|V(G)/_{\equiv_G}| \leq |V(G)|$. Note also that we can ensure that $\F$ is closed under isomorphism by using the same idea as in the proof of Theorem~\ref{the:characterization}.
Finally, we define the filtering function $f$ in such a way that, for every $i \in \{0,\ldots,m-1\}$, $f(i) = i+1$. 
It is now not difficult to verify that indeed $\rank(\Ce) \cap \vcg^2 = \rank(\subc{\F,f}) \cap \vcg^2$, and the claim follows.
\end{proof}

As discussed above, $\con\closeness$ is not a subgraph-based measure (this is implicit in the proof of Proposition~\ref{pro:closeness}), whereas $\con\closeness$ is a subgraph-based measure relative to the induced ranking (follows from Theorem~\ref{the:ranking-connected-graphs}). This reveals a striking difference between the two notions of expressiveness, that is, being subgraph-based or being subgraph-based realtive to the induced ranking, when focussing on connected graphs.

We conclude this section by stressing that Theorem~\ref{the:ranking-connected-graphs} provides a unifying framework for all centrality measures in a practically relevant setting: connected graphs and induced ranking. Indeed, graphs in real-life scenarios, although might be non-connected, they typically consists of one dominant connected component and several small components that are usually neglected as, by default, the most important vertex appears in the dominant component. Moreover, in real-life graph-based applications, we are typically interested in the induced ranking rather than the absolute centrality values assigned to vertices.
\section{Monotonic Filtering Functions}\label{sec:monotonic-functions}

Until now, we considered arbitrary filtering functions without  any restrictions. On the other hand, the filtering functions $f_{\times 2}$ and $\log_2$ used to express $\stress$ and $\allsub$, respectively, as subgraph-based measures are monotonic; formally, a filtering function $f$ is {\em monotonic} if, for all $x,y \in \mathbb{N}$, $x \leq y$ implies $f(x) \leq f(y)$.
It is natural to ask Questions I and II for {\em monotonic subgraph-based centrality measures}, i.e., subgraph-based centrality measures $\subc{\F,f}$ where $f$ is monotonic.
Needless to say, one can study a plethora of different families of subgraph-based centrality measures that use filtering functions with certain properties (e.g., linear functions, logarithmic functions, etc.). However, such a thorough analysis is beyond the scope of this work, and it remains the subject of future research.
%
%In what follows, we revisit Questions I and II focussing on the family of subgraph motif centrality measures that use monotonic filtering functions.

\medskip

\noindent
\paragraph{Monotonic Subgraph-based Measures.}
We first give a result analogous to Proposition~\ref{pro:no-subc-measure}, showing that not all subgraph-based measures are monotonic, and thus, the bounded value property is not the answer to Question I in the case of monotonic subgraph-based measures.
%
\begin{comment}
In particular, with $\hat{G} = (\{v_1,v_2,v_3\},\{\{v_2,v_3\}\})$, for the centrality measure $\Ce$ defined as
\[
\Ce(v,G)\ =\ \begin{cases}
1 & G \neq \hat{G} \text{ or } (G = \hat{G} \text{ and } v=v_3)\\
3 & G = \hat{G} \text{ and } v=v_1\\
2 & G = \hat{G} \text{ and } v=v_2
\end{cases}
\]
we can show that it is subgraph motif, but, for every subgraph family $\F$ and filtering function $f$ such that $\Ce = \subc{\F,f}$, it holds that $f$ is not monotonic. We therefore get that:
\end{comment}

\begin{proposition}\label{pro:no-monotonic-subc-measure}
	There is a subgraph-based centrality measure that is not monotonic.
\end{proposition}

\begin{proof}
	Let $G_1$ be the graph with just one isolated node $(\{v_1\},\emptyset)$, and $G_2$ be the graph $(\{v_1,v_2,v_3\},\{\{v_2,v_3\}\})$.
	Consider the (partial) function $\Ce : \vg\ \rightarrow\ \mathbb{R}$ defined as follows:
	\[
	\Ce(v,G)\ =\ \begin{cases}
	1 & G = G_2 \text{ and } v \in \{v_2,v_3\} \\
	2 & G = G_2 \text{ and } v=v_1\\
	3 & G = G_1 \text{ and } v=v_1.
	\end{cases}
	\]
	It is easy to see that $\Ce$ can be extended to a proper centrality measure $\hat{\Ce}$: for every pair $(u,G') \in \vg$ such that $(v,G) \simeq (u,G')$, where $(v,G)\in \{(v_1, G_1),(v_1,G_2),(v_2,G_2),(v_3,G_2)\}$, let $\hat{\Ce}(u,G') = \Ce(v,G)$, and in any other case let $\hat{\Ce}(u,G') = 1$. 
	We first show that $\hat{\Ce}$ is a subgraph-based measure. Notice that, for every vertex $v \in \ins{V}$, $\hat{\Ce}(v,G_v) = \Ce(v_1,G_1) = 3$. Hence, we have only two options concerning the set of connected subgraphs assigned to the vertices of $G_2$ by a subgraph family, and the filtering function, which are the following: with $G_{uv} $ being the single-edge graph $(\{u,v\},\{\{u,v\}\})$, either
	\[
	\F_1(v,G)\ =\ \begin{cases}
	\emptyset & v=v_1 \text{ and } G = G_1\\
	\{G_{v_1}\} & v=v_1 \text{ and } G = G_2\\
	\{G_{v_2},G_{v_2v_3}\} & v=v_2 \text{ and } G=G_2\\
	\{G_{v_3},G_{v_2v_3}\} & v=v_3 \text{ and } G=G_2
	\end{cases}
	\]
	with $f_1(0)=3$, $f_1(1)=2$ and $f_1(2)=1$, or
	\[
	\F_2(v,G)\ =\ \begin{cases}
	\{G_{v_1}\} & v=v_1 \text{ and } G = G_1\\
	\emptyset & v=v_1 \text{ and } G = G_2\\
	\{G_{v_2},G_{v_2v_3}\} & v=v_2 \text{ and } G=G_2\\
	\{G_{v_3},G_{v_2v_3}\} & v=v_3 \text{ and } G=G_2
	\end{cases}
	\]
	with $f_2(0)=2$, $f_2(1)=3$ and $f_2(2)=1$. 
	We can now extend $\F_1$ and $\F_2$ into subgraph families that are closed under isomorphism as follows: for every $(u,G') \in \vg$ with $(v,G) \simeq (u,G')$, if $(v,G)\in \{(v_1, G_1),(v_1,G_2),(v_2,G_2),(v_3,G_2)\}$, then $\F_1(v,G) \simeq \F_1(u,G')$ and $\F_2(v,G) \simeq \F_2(u,G')$, otherwise, $\F_1(u,G') = \emptyset$ and $\F_2(u,G') = \{G_u\}$. It is clear that $\hat{\Ce} = \subc{\F_1,f_1} = \subc{\F_2,f_2}$. Observe, however, that both $f_1$ and $f_2$ are not monotonic functions.
\end{proof}

The proof of Proposition~\ref{pro:no-monotonic-subc-measure} essentially tells us that the key reason why the subgraph-based measure $\hat{\Ce}$ is not monotonic is because the maximum centrality value is assigned to a vertex surrounded by few connected subgraphs.
To formalize this intuition, we first collect all the different values $x$ assigned by a measure $\Ce$ to the vertices of a graph $G$ that are surrounded by ``too many'' connected subgraphs such that $x$ does not exceed the maximum value assigned by $\Ce$ to the vertices of $G$ surrounded by ``too few'' connected subgraphs. More precisely, for an integer $n >0 $, we define the set of values
\[
\bval{G}{n}(\Ce)\ =\ \left\{x \in \bigcup_{m>0} \val{G}{m}(\Ce) \mid
x \not\in \val{G}{n}(\Ce)~~\text{and}~~x < \max \val{G}{n}(\Ce)\right\}.
\]
We then define the set of values
\[
\bval{}{n}(\Ce)\ =\ \bigcup_{G \in \gr} \bval{G}{n}(\Ce).
\]
We can now define a refined version of the bounded value property, which provides a better upper bound for $|\val{}{n}(\Ce)|$:

\begin{definition}[\textbf{Monotonic Bounded Value Property}]\label{def:mbvp}
	A centrality measure $\Ce$ enjoys the {\em monotonic bounded value property} if, for every $n>0$, $|\val{}{n}(\Ce)| \leq n+1-|\bval{}{n}(\Ce)|$. \hfill\markfull
\end{definition}

It is not difficult to see that the measure $\Ce$ devised in the proof of Proposition~\ref{pro:no-monotonic-subc-measure} does not enjoy the monotonic bounded value property. Indeed, $\val{}{1}(\Ce) = \{1,3\}$ and $\bval{}{1} = \{2\}$, and thus, $|\val{}{1}(\Ce)| =  2 > 1$.
The above refinement of the bounded value property is all we need to get a precise characterization of monotonic subgraph-based measures; hence the name ``monotonic bounded balue property''.

\begin{theorem}\label{the:monotonic-characterization}
	Consider a centrality measure $\Ce$. The following statements are equivalent:
	
	\begin{enumerate}
		\item $\Ce$ is a monotonic subgraph-based centrality measure.
		\item $\Ce$ enjoys the monotonic bounded value property.
	\end{enumerate}
\end{theorem}

\medskip

\noindent
\paragraph{Induced Ranking.}
Concerning the expressiveness of monotonic subgraph-based centrality measures relative to the induced ranking, we can show that the non-uniform version of the monotonic bounded value property provides a precise characterization.

\begin{definition}[\textbf{Non-Uniform Monotonic Bounded Value Property}]\label{def:non-uniform-mbvp}
	A centrality measure $\Ce$ enjoys the {\em non-uniform monotonic bounded value property} if, for every integer $n>0$ and graph $G \in \gr$, it holds that $|\val{G}{n}(\Ce)| \leq n+1-|\bval{G}{n}(\Ce)|$. \hfill\markfull
\end{definition}

We can then establish the following characterization that is in striking difference with Theorem~\ref{the:characterization-ranking}, which shows that the non-uniform bounded value property is only a necessary condition (but not a sufficient condition) for a centrality measure being subgraph-based relative to the induced ranking.

\begin{theorem}\label{the:monotonic-characterization-ranking}
	Consider a centrality measure $\Ce$. The following statements are equivalent:
	
	\begin{enumerate}
		\item $\Ce$ is a monotonic subgraph-based centrality measure relative to the induced ranking.
		\item $\Ce$ enjoys the non-uniform monotonic bounded value property.
	\end{enumerate}
\end{theorem}

\medskip

\noindent
\paragraph{Connected Graphs.} Recall that the family of subgraph-based measures relative to the induced ranking provides a unifying framework for all centrality measures whenever we concentrate on connected graphs (see Theorem~\ref{the:ranking-connected-graphs}). Interestingly, a careful inspection of the proof of Theorem~\ref{the:ranking-connected-graphs} reveals that this holds even for the family of monotonic subgraph-based measures relative to the induced ranking.

\begin{theorem}\label{the:ranking-connected-graphs-monotonic}
	Consider a centrality measure $\Ce$. It holds that $\conC$ is a monotonic subgraph-based measure relative to the induced ranking.
\end{theorem}
\section{Classification}\label{sec:classification}

We proceed to
%The results of the previous sections provide tools that allow us to establish whether a measure can be expressed as a (monotonic) subgraph motif measure (relative to the induced ranking).
%
%We use these tools to 
determine whether existing measures belong to the family of (monotonic) subgraph-based measures (relative to the induced ranking) by exploiting the technical tools provided by the results of the previous sections.
Such a classification, apart from being interesting in its own right, will provide insights on the structural similarities and differences among existing centrality measures.
To this end, we focus on established measures from the literature 
%(including $\stress$, $\allsub$, and $\closeness$), 
and provide a rather complete classification depicted in Tables~\ref{tab:classification} and~\ref{tab:classification-monotonic}; due to space constraints, the formal definitions of the considered measures are omitted.
%and can be found in the appendix.
%
The second (resp., third) column determines whether the measure $\Ce$ stated in the first column is subgraph-based (resp., subgraph-based relative to the induced ranking); $\checkmark$ means that it is,  $\times$ means that it is not, $\times[\mathit{trees}]$ means that it is not even for trees, $\checkmark[\mathit{con}]$ means that it is over connected graphs, $\checkmark[\mathit{trees}]$ means that it is over trees, and $?$ means that it is open.
Concerning Table~\ref{tab:classification-monotonic}, $\star$ refers to any measure considered in Table~\ref{tab:classification} apart from $\betweeness$, and $\times[\mathit{con}]$ means that the respective measure (i.e., $\betweeness$) is not monotonic subgraph-based even for connected graphs.
Note that Table~\ref{tab:classification-monotonic} is identical to Table~\ref{tab:classification}, apart from $\betweeness$, which is provably not monotonic subgraph-based (relative to the induced ranking).

We would like to remark that the result $\checkmark[\mathit{con}]$ for $\eigenvector$ in both tables holds for a broader class of graphs than connected graphs. 
Moreover, we can show that $\betweeness$ is a (monotonic) subgraph-based measure (relative to the induced ranking) for a class of graphs that captures the class of trees and is incomparable to the class of connected graphs.
For the sake of readability, we state our expressibility results only for trees and connected graphs.
%further details can be found in the appendix.

\begin{table}
	\centering
	\begin{tabular}{c||c|c}
		\textbf{Measure} & \textbf{Absolute Values} & \textbf{Induced Ranking} \\ \hline
		$\stress$ & $\checkmark$ & $\checkmark$ \\ \hline
		$\allsub$ & $\checkmark$ & $\checkmark$ \\ \hline
		$\degree$ & $\checkmark$ & $\checkmark$ \\ \hline
		$\clique$ & $\checkmark$ & $\checkmark$ \\ \hline
		$\closeness$ & $\times[\mathit{trees}]$ & $\times$ and $\checkmark[\mathit{con}]$ \\ %\midrule
		%$\alpha$-$\decay$, $\alpha \in (0,1)$ & $\times[\mathit{trees}]$ & $\times$ and $\checkmark[\mathit{con}]$ \\ 
		\hline
		$\harmonic$ & $\times[\mathit{trees}]$ & $\times$ and $\checkmark[\mathit{con}]$ \\ \hline
		$\pagerank$ & $\times[\mathit{trees}]$ & $\times$ and $\checkmark[\mathit{con}]$ \\ \hline
		$\eigenvector$ & $\times[\mathit{trees}]$ & $?$ and $\checkmark[\mathit{con}]$ \\ \hline
		$\betweeness$ & $?$ and $\checkmark[\mathit{trees}]$  & $?$ and $\checkmark[\mathit{con}]$ \\ \hline
	\end{tabular}
	\caption{Subgraph-based Measures
		%The second (resp., third) column determines whether the measure $\Ce$ stated in the first column is subgraph motif (resp., subgraph motif relative to the induced ranking); $\checkmark$ means that it is,  $\times$ means that it is not, $\times[\mathit{trees}]$ means that it is not even for trees, $\checkmark[\mathit{con}]$ means that it is over connected graphs, $\checkmark[\mathit{trees}]$ means that it is over trees, and $?$ means that it is open.
	}
	\label{tab:classification}
	\vspace{-2mm}
\end{table}

\begin{table}
	\centering
	\begin{tabular}{c||c|c}
		\textbf{Measure} & \textbf{Absolute Values} & \textbf{Induced Ranking} \\ \hline
		$\star$ & as in Table~\ref{tab:classification} & as in Table~\ref{tab:classification}  \\ \hline
		$\betweeness$ & $\times[\mathit{con}]$ \text{and} $\checkmark[\mathit{trees}]$  & $\times$ \text{and} $\checkmark[\mathit{con}]$ \\ \hline
	\end{tabular}
	\caption{Monotonic Subgraph-based Measures
		%By $\star$ we mean any measure considered in Table~\ref{tab:classification} apart from $\betweeness$. This table is identical to Table~\ref{tab:classification}, apart from $\betweeness$, which is provably not monotonic subgraph motif (relative to the induced ranking). The notation $\times[\mathit{con}]$ follows the same logic as in Table~\ref{tab:classification}.
	} 
	\label{tab:classification-monotonic}
	\vspace{-5mm}
\end{table}

\medskip

\noindent\paragraph{Take-home Messages.}
We highlight the key take-home messages of the above classification, which we believe provide further insights concerning the centrality measures in question:

\begin{enumerate}
	\item If we focus on the induced ranking rather than the absolute values over connected graphs, then the family of monotonic subgraph-based measures should be understood as a unifying framework that incorporates every other measure.
	
	\item Our classification excludes a priori the adoption of certain centrality measures (e.g., $\closeness$, $\harmonic$, etc.) in applications where the importance of a vertex should be measured based on the connected subgraphs surrounding it.
	
	\item $\betweeness$, which computes the percentage of the shortest paths in a graph going through a vertex, is of different nature compared to all the other measures. Notably, although it looks similar to $\stress$, it behaves in a significantly different way. The relationship of $\betweeness$ with (monotonic) subgraph-based measures deserves further investigation.
	
	\item There is a notable difference between the two feedback measures considered in our classification, namely $\pagerank$ and $\eigenvector$, that deserves further exploration. As mentioned above, $\eigenvector$ is a (monotonic) subgraph-based measure relative to the induced ranking over a broader class $\mathcal{C}$ of graphs than connected graphs, whereas $\pagerank$ is provable {\em not} a subgraph-based measure over the class $\mathcal{C}$.
\end{enumerate}

\medskip
\noindent\paragraph{A Note on Directed Graphs.} As discussed in the clarification remark at the end of the Introduction, although our analysis (including the classification of this section) focused on undirected graphs, all the notions and results can be transferred to directed graphs under the notion of weak connectedness. The only exception is the negative result $\times[\mathit{trees}]$ for $\eigenvector$ in Tables~\ref{tab:classification} and~\ref{tab:classification-monotonic}. Although we can show that for directed graphs, $\eigenvector$ is not a (monotonic) subgraph-based centrality measure, it remains open whether this holds even for directed trees (i.e., directed graphs whose underlying undirected graph is a tree).

\section{Conclusions}\label{sec:conclusions}

We have provided a rather complete picture concerning the absolute expressiveness of the family of (monotonic) subgraph-based centrality measures (relative to the induced ranking) by establishing precise characterizations. We have also presented a detailed classification of standard centrality measures by using the tools provided by the aforementioned characterizations.
Although our development focused on undirected graphs, all the notions and results can be transferred to directed graphs under the standard notion of weak connectedness.
%
%In the case of (monotonic) subgraph motif measures, as well as monotonic subgraph motif measure relative to the induced ranking, the obtained characterizations rely on bounded-value-like properties. On the other hand, in the case of subgraph motif measures relative to the induced ranking, the obtained characterization relies on graph colorings, and we only have a necessary condition based on a bounded-value-like property.

We would like to stress that the machinery on graph colorings, introduced in Section~\ref{sec:characterization-ranking}, can be used to provide characterizations for all the families considered in the paper, and not only for the family of subgraph-based measures relative to the induced ranking. For example, we can show that a measure $\Ce$ is subgraph-based iff there exists a precoloring of $\vg$ that is uniformly $\Ce$-injective; the latter is defined as non-uniform $\Ce$-injectivity with the difference that $\Ce$-injectivity is enforced across all the graphs (not only inside a certain graph). 
%Analogous properties over precolorings that lead to characterizations for monotonic subgraph motif measures (relative to the induced ranking) can be defined.

The obvious question that remains open is whether we can isolate a bounded-value-like property that characterizes subgraph-based measures relative to the induced ranking. We believe that our coloring-based characterization (Theorem~\ref{the:monotonic-characterization-ranking}) is a useful tool towards such a bounded-value-like characterization.
Finally, towards a deeper understanding of subgraph-based measures, one should perform a more refined analysis by focussing on restricted classes of subgraph families and filtering functions that enjoy desirable structural properties.

%%
%% Bibliography
%%

%% Please use bibtex, 

\bibliography{references}

%\newpage
%\appendix

%\section{Preliminary Notions}
%\input{../Appendix/app-preliminaries}
%\section{Proofs of Section~\ref{sec:characterization}}
%\input{../Appendix/app-sec4}
%\section{Proofs of Section~\ref{sec:characterization-ranking}}
%\input{../Appendix/app-sec5}
%\section{Proofs of Section~\ref{sec:monotonic-functions}}
%\input{../Appendix/app-sec6}
%\section{Proofs of Section~\ref{sec:classification}}
%\input{../Appendix/app-sec7}
%\section{Directed Graphs}
%\input{../Appendix/app-directed}

\end{document}